\documentclass{article}

\usepackage{microtype}
\usepackage{graphicx}
\usepackage{subfigure}
\usepackage{booktabs} 

\usepackage{hyperref}

\usepackage[accepted]{icml2023}

\usepackage{amsmath}
\usepackage{amssymb}
\usepackage{mathtools}
\usepackage{amsthm}

\usepackage[capitalize,noabbrev]{cleveref}

\usepackage{bbm}
\usepackage{xspace}
\usepackage{multirow}
\usepackage{mathtools}

\newcommand{\gcnmodelname}{SPA\xspace}
\newcommand{\transmodelname}{TransFlood\xspace}
\newcommand{\gcnfloodmodelname}{GCNFlood\xspace}
\DeclareMathOperator*{\argmin}{arg\,min}

\theoremstyle{plain}
\newtheorem{theorem}{Theorem}[section]

\theoremstyle{definition}

\theoremstyle{remark}
\newtheorem{remark}[theorem]{Remark}

\usepackage[textsize=tiny]{todonotes}

\icmltitlerunning{A Similarity Flooding Perspective for Multi-sourced Knowledge Graph Embeddings}

\begin{document}

\twocolumn[
\icmltitle{What Makes Entities Similar? A Similarity Flooding Perspective for Multi-sourced Knowledge Graph Embeddings}



\icmlsetsymbol{equal}{*}

\begin{icmlauthorlist}
\icmlauthor{Zequn Sun}{skl}
\icmlauthor{Jiacheng Huang}{skl}
\icmlauthor{Xiaozhou Xu}{alibaba}
\icmlauthor{Qijin Chen}{alibaba}
\icmlauthor{Weijun Ren}{alibaba}
\icmlauthor{Wei Hu}{skl,hds}
\end{icmlauthorlist}

\icmlaffiliation{skl}{State Key Laboratory for Novel Software Technology, Nanjing University, Nanjing, China}
\icmlaffiliation{alibaba}{Alibaba Group, Hangzhou, China}
\icmlaffiliation{hds}{National Institute of Healthcare Data Science, Nanjing University, Nanjing, China}

\icmlcorrespondingauthor{Wei Hu}{whu@nju.edu.cn}

\icmlkeywords{Knowledge Graphs, Entity Alignment, Similarity Flooding, Fixpoint}
\vskip 0.3in
]



\printAffiliationsAndNotice{}  

\begin{abstract}
Joint representation learning over multi-sourced knowledge graphs (KGs) yields transferable and expressive embeddings that improve downstream tasks. Entity alignment (EA) is a critical step in this process. Despite recent considerable research progress in embedding-based EA, how it works remains to be explored. In this paper, we provide a similarity flooding perspective to explain existing translation-based and aggregation-based EA models. We prove that the embedding learning process of these models actually seeks a fixpoint of pairwise similarities between entities. We also provide experimental evidence to support our theoretical analysis. We propose two simple but effective methods inspired by the fixpoint computation in similarity flooding, and demonstrate their effectiveness on benchmark datasets. Our work bridges the gap between recent embedding-based models and the conventional similarity flooding algorithm. It would improve our understanding of and increase our faith in embedding-based EA.
\end{abstract}

\section{Introduction}

A knowledge graph (KG) is a set of relational triplets.
Each triplet is in the form of (\textit{subject entity}, \textit{relation}, \textit{object entity}), denoted by $(s,r,o)$ for short.
A relational triplet indicates a relation between two entities, such as (\textit{ICML 2023}, \textit{hosted in}, \textit{Hawaii}).
Different KGs are created by harvesting various webs of data.
They could cover complementary knowledge from different sources and thus aid in resolving the incompleteness issue of each single KG.
In recent years, representing multi-sourced KGs in a unified embedding space, as illustrated in Figure~\ref{fig:example}, has shown promising potential in promoting knowledge fusion and transfer \cite{LinkNBed}.
It uses entity alignment (EA) between different KGs to jump-start joint and transferable representation learning.
EA refers to the match of identical entities from different KGs, such as ``\textit{ICML}'' and ``\textit{International Conference on Machine Learning}''.
The goal of multi-sourced KG embedding is learning to distinguish between identical and dissimilar entities in different KGs while capturing their respective graph structures.
By aligning the embeddings of identical entities, an entity in one KG can indirectly capture the graph structures of its counterpart in another KG, resulting in more informative representations to benefit downstream tasks.

Therefore, as a fundamental task, embedding-based EA has drawn increasing attention \cite{MTransE,RSN,OpenEA,EA_TKDE,EA_survey,EA_VLDBJ,NeoEA}.
The key of embedding-based EA lies in how to generate entity embeddings for alignment learning.
Existing techniques fall into two groups, translation-based \cite{MTransE,JAPE,TransEdge} and aggregation-based models.
A translation-based model adopts TransE \cite{TransE} or its variants for embedding learning.
Given a triplet $(s,r,o)$,
TransE interprets a relation embedding as the translation vector from the subject entity embedding to the object entity.
Another group of EA models uses graph convolutional networks (GCNs) \cite{GCN} to generate an entity representation by aggregating its neighbor embeddings.

\begin{figure}[t]
\centering
\includegraphics[width=0.81\linewidth]{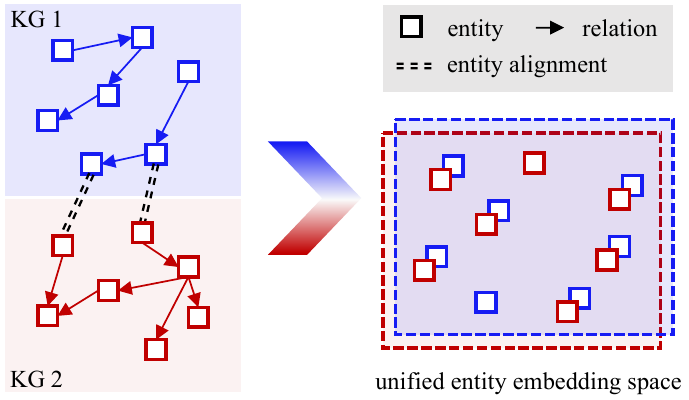}
\caption{Illustration of representing two KGs in a unified space.}
\label{fig:example}
\end{figure}

Despite the considerable technical progress in embedding-based EA, a critical question remains unanswered, 
i.e., \textit{what makes entity embeddings similar in an EA model?}
This question may also cause some researchers to misunderstand and distrust embedding-based EA techniques.
Besides, the connection between embedding-based EA models and traditional symbolic methods remains unexplored.
Under the circumstances, we seek to answer the question. 
We present a similarity flooding (SF) perspective to understand and improve embedding-based EA with both theoretical analysis and experimental evidence.
SF is a widely-used algorithm for matching structured data \cite{Similarity_flooding,Similarity_flooding_report}.
We show that the essence of recent embedding-based EA is also a variant of SF, and learning embeddings is only a means.

Our main contributions are summarized as follows:
\begin{itemize}
\setlength{\itemsep}{1pt}
\setlength{\topsep}{1pt}
\item We present the first theoretical analysis of embedding-based EA techniques to understand how they work.
We provide a similarity flooding perspective to unify the basic translation- and aggregation-based EA models. 
We also build a close connection between embedding-based and traditional symbolic-based EA via the unified perspective of fixpoint computation for entity similarities.
This work would improve our understanding of and increase our faith in embedding-based models.

\item We propose two simple but effective methods based on our theoretical analysis to improve EA. 
The first is a variant of similarity flooding that computes the fixpoint of entity similarities using the entity compositions induced from TransE or GCN. 
This method does not need to learn KG embeddings. 
The second is inspired by the fact that the similarity fixpoint indicates an embedding fixpoint. 
It introduces a self-propagation connection in neighborhood aggregation to let entity embeddings have a chance of propagating back to themselves. 

\item We conduct experiments on DBP15K \cite{JAPE} and OpenEA \cite{OpenEA} to validate the effectiveness of our EA methods and provide experimental evidence to support our theoretical conclusions.
The source code is available at our GitHub repository\footnote{\scriptsize \url{{https://github.com/nju-websoft/Unify-EA-SF}}}.
\end{itemize}

\section{Preliminaries}\label{sect:preliminary}
We first introduce the EA task, and then discuss the basic translation-based and aggregation-based models.
We would like to know how to represent an entity in these models so that we can learn more about what factors influence entity similarities.
Finally, we introduce the SF algorithm.

\subsection{Problem Definition}
Formally, let $\mathcal{X}$ and $\mathcal{Y}$ be the entity sets of the source and target KGs, respectively. 
In the supervised setting, we are given a set of seed entity alignment pairs $\mathcal{A}$ as training data.
For an aligned entity pair $(x, y) \in \mathcal{A}$, where $x\in \mathcal{X}$ and $y\in \mathcal{Y}$,
the KG embeddings for EA are expected to hold:
$
\mathbf{x}=\argmin_{x'\in \mathcal{X}} \pi(\mathbf{x}',\mathbf{y}).
$
Hereafter, we use boldface type to denote vector embeddings, e.g., $\mathbf{x}$ and $\mathbf{y}$ for the embeddings of $x$ and $y$, respectively.
$\pi(\mathbf{x},\mathbf{y})$ is a distance measure. 
In this paper, we consider the Euclidean distance, i.e., $\pi(\mathbf{x},\mathbf{y})=||\mathbf{x}-\mathbf{y}||_2$, where $||\cdot||_2$ denotes the $L_2$ vector norm.
It indicates that, if $x$ and $y$ are aligned entities, $\mathbf{y}$ is expected to be the nearest cross-KG neighbor of $\mathbf{x}$ in the embedding space.
To achieve this goal, given a small set of seed alignment, $\mathcal{A}\subset\{(x,y)|x \equiv y\}$, as training data, the general objective of alignment learning is to minimize the embedding distance of entity pairs in $\mathcal{A}$ \cite{MTransE}:
\begin{equation} 
\label{eq:align_learning}
\min_{(x,y)\in \mathcal{A}} \pi(\mathbf{x},\mathbf{y}).
\end{equation}
Although many models introduce various negative sampling methods \cite{BootEA} to generate dissimilar entity pairs and learn to separate the embeddings of dissimilar entities, Eq.~(\ref{eq:align_learning}) is the most common and indispensable learning objective,
which is our focus in this paper.

\subsection{TransE-based EA} 
The typical learning objective of TransE-based models is to solve two optimization problems, i.e., translational embedding learning and alignment learning, as shown in Eqs.~(\ref{eq:align_learning}) and (\ref{eq:transe}), respectively.
\begin{equation} 
\label{eq:transe}\small
\min_{(s,r,o)\in \mathcal{T}} \|\mathbf{s}+\mathbf{r}-\mathbf{o}\|_2^2, \ \ \ \text{s.t.} \ \ \|\mathbf{e}\|_2^2=1, \, \forall \, e \in \mathcal{X}\cup\mathcal{Y},
\end{equation}
where $\mathcal{T}$ is the set of triplets and $e$ denotes an entity.
Considering that the two optimization problems have a trivial optimal solution with all entities and relations having zero vectors, most EA models normalize each entity embedding to a unit vector.
Therefore, we further introduce a Lagrange term $\lambda_e \sum_{e\in\mathcal{X}\cup\mathcal{Y}} (||\mathbf{e}||_2^2 - 1)$, where $\lambda_e$ is the Lagrange multiplier.
The combined optimization problem is
\begin{equation}
\label{eq:transe_lagrange}
\resizebox{.9\columnwidth}{!}{$
\mathcal{L}(\mathbf{\Theta}) = \smashoperator[r]{\sum_{(s,r,o)\in \mathcal{T}}} \|\mathbf{s}+\mathbf{r}-\mathbf{o}\|_2^2 + \smashoperator[r]{\sum_{(x,y)\in \mathcal{A}}} \|\mathbf{x}-\mathbf{y}\|_2^2 + \lambda_e \smashoperator[r]{\sum_{e\in\mathcal{X}\cup\mathcal{Y}}} (||\mathbf{e}||_2^2 - 1),
$}
\end{equation}
where $\mathbf{\Theta}$ denotes the entity and relation embeddings.
The optimization problem then shifts to solving the following equation:
$\bigtriangledown_{\mathbf{\Theta},\lambda_e} \mathcal{L}(\mathbf{\Theta}) = \mathbf{0}$.
Then, we can derive the representations of relations and entities in the model.

\noindent\textbf{Deriving relation representations.}
We first consider relation embeddings and take the relation $r$ as an example.
We are interested in the gradients of the loss in Eq.~(\ref{eq:transe_lagrange}) with respect to $r$: $\bigtriangledown_{\mathbf{r}} \mathcal{L}(\mathbf{\Theta}) = \bigtriangledown_{\mathbf{r}} \sum_{(s,r,o)\in \mathcal{T}_r} \|\mathbf{s}+\mathbf{r}-\mathbf{o}\|_2^2$,
where $\mathcal{T}_r$ denotes the set of triplets involving $r$.
Letting the above derivative be zero, we can derive
\begin{equation}
\label{eq:r}\small
\mathbf{r} = \frac{1}{|\mathcal{T}_r|} \sum_{(s,r,o)\in \mathcal{T}_r} (\mathbf{o} - \mathbf{s}).
\end{equation}
The equation aligns with the motivation of TransE that represents a relation as the translation vector between its subject and object entity embeddings.
Given this equation, we can use the final entity embeddings to represent a relation.

\noindent\textbf{Deriving entity representations.}\label{sect:entity_representation}
An entity may appear as the subject or object in a triplet.
To simplify the formulations without information loss, we introduce reverse triplets following the convention in KG embedding models~\cite{RSN}.
For each triplet $(s,r,o)$, we add a new triplet $(o,r^{-1},s)$, where $r^{-1}$ denotes the reverse relation of $r$.
In this way, we only need to consider the outgoing edges of an entity, i.e., the triplets with the given entity as the subject.
The original ingoing edges are considered by including their reverse edges.
We use $\mathcal{T}_e$ to denote the triplets with $e$ as the subject.
Specifically, given entity $e$, we are interested in the gradients of the loss in Eq.~(\ref{eq:transe_lagrange}) with respect to embedding $\mathbf{e}$, i.e.,
$\bigtriangledown_{\mathbf{e}} \mathcal{L}(\mathbf{\Theta}) =\bigtriangledown_{\mathbf{e}}\sum_{(e,r,o)\in \mathcal{T}_e} \|\mathbf{e}+\mathbf{r}-\mathbf{o}\|_2^2 + \mathbbm{1}_{\exists (e,\hat{e})\in\mathcal{A}} \bigtriangledown_{\mathbf{e}} \,\|\mathbf{e}-\hat{\mathbf{e}}\|_2^2 + \lambda_e \bigtriangledown_{\mathbf{e}} (||\mathbf{e}||_2^2 - 1)$,
where $\mathbbm{1}$ is an indicator function.
By setting the gradients to be zero vectors, we obtain $\mathbf{e} = \frac{1}{|\mathcal{T}_e|+\lambda_e} \sum_{(e,r,o)\in \mathcal{T}_e} (\mathbf{o} - \mathbf{r}) + \mathbbm{1}_{\exists (e,\hat{e})\in\mathcal{A}} (\mathbf{e}-\hat{\mathbf{e}})$.
With proper EA training strategies, e.g., parameter sharing \cite{OpenEA}, $e$ and $\hat{e}$ would have the same embeddings, i.e., $\mathbf{e}-\hat{\mathbf{e}}=\mathbf{0}$.
In addition, 
we can apply normalization to $\mathbf{e}$ to ensure $\| \mathbf{e}\|=1$, and then we can 
replace $|\mathcal{T}_e|+\lambda_e$ with $|\mathcal{T}_e|$.
Thus, we obtain $\mathbf{e} = \frac{1}{|\mathcal{T}_e|} \sum_{(e,r,o)\in \mathcal{T}_e} (\mathbf{o} - \mathbf{r})$.
Note that, in this equation, we still need relation embeddings to represent an entity.
To get free of relation embeddings, we can replace them with the composition of related subject and object entity embeddings as shown in Eq.~(\ref{eq:r}), and get
\begin{equation}\label{eq:transe_ent_embed_final}\small
\mathbf{e} = \frac{1}{|\mathcal{T}_e|} \sum_{(e,r,o)\in \mathcal{T}_e} \Big(\mathbf{o} - \frac{1}{|\mathcal{T}_r|} \sum_{(s',r,o')\in \mathcal{T}_r} (\mathbf{o}' - \mathbf{s}')\Big) .
\end{equation}
In this way, we represent an entity by the composition of its related entities in the same KG.

\subsection{GCN-based EA} 
In a CGN-based EA method, an entity is first represented by aggregating its neighbors. 
For brevity, we consider a one-layer GCN layer \cite{GCN} with mean-pooling as the aggregation function, i.e., $G(x)=\frac{1}{|N(x)|}\sum_{x'\in N(x)}\mathbf{x}'$.
The entity representation in GCNs is:
\begin{equation}\label{eq:gcn_e}\small
\mathbf{e}=\frac{1}{|N(e)|}\sum_{e'\in N(e)}\mathbf{e}'.
\end{equation}
Then, given the output representations, we minimize the embedding distance of identical entities in seed entity alignment for alignment learning, as shown in Eq.~(\ref{eq:align_learning}).
Finally, we use $k$NN search to find the counterpart for a given entity.

\subsection{Similarity Flooding}

Similarity flooding \cite{Similarity_flooding} is an iterative graph matching technique based on fixpoint computation.
It is a fundamental algorithm and widely used in a variety of graph matching contexts, such as ontology mapping and database schema matching \cite{ShvaikoE13}.
Given two input graphs $G_1$ and $G_2$ with the aim of finding the mapping of identical nodes, the similarity flooding algorithm first creates a pairwise connectivity graph (PCG), which is an auxiliary data structure for similarity propagation.
As shown in Figure~\ref{fig:pcg}, in a PCG, a node is an entity pair $(x_1,y_1)$ with the similarity of $\sigma(x_1,y_1)$ (called a mapping pair), where the two entities are from the two graphs, respectively, i.e., $x_1\in G_1$ and $y_1\in G_2$.
An edge $\big((x_1,y_1), r_1, (x_2,y_2)\big)$ of the PCG is induced from the two graphs having $(x_1,r_1,x_2)\in G_1$ and $(y_1,r_1,y_2)\in G_2$.
The relation $r_1$ would be further given a weight, called the propagation coefficient, which ranges from $0$ to $1$ and can be computed in different ways \cite{Similarity_flooding_report}.
The directed weighted edge $\big((x_1,y_1), r_1, (x_2,y_2)\big)$ indicates how well the similarity of $(x_1,y_1)$ propagates to its neighbor $(x_2,y_2)$.
Then, the algorithm propagates the similarity of each node (i.e., mapping pair) over the PCG using fixpoint computation and finally outputs the node mappings.
The fixpoint formula for similarity flooding is
\begin{equation}\label{eq:fixpoint}
\Omega = \texttt{normalize}\big(\Omega_0 + \Omega + \varphi(\Omega_0 + \Omega)\big),
\end{equation}
where $\Omega_0$ is the node similarity matrix,
and $\varphi$ is the propagation function.
In conventional graph matching methods, $\Omega_0$ can be computed by string matching.
In our work, we follow the supervised setting of embedding-based EA, and use seed entity alignment to initialize $\Omega_0$.

\begin{figure}[t]
\centering
\includegraphics[width=0.87\linewidth]{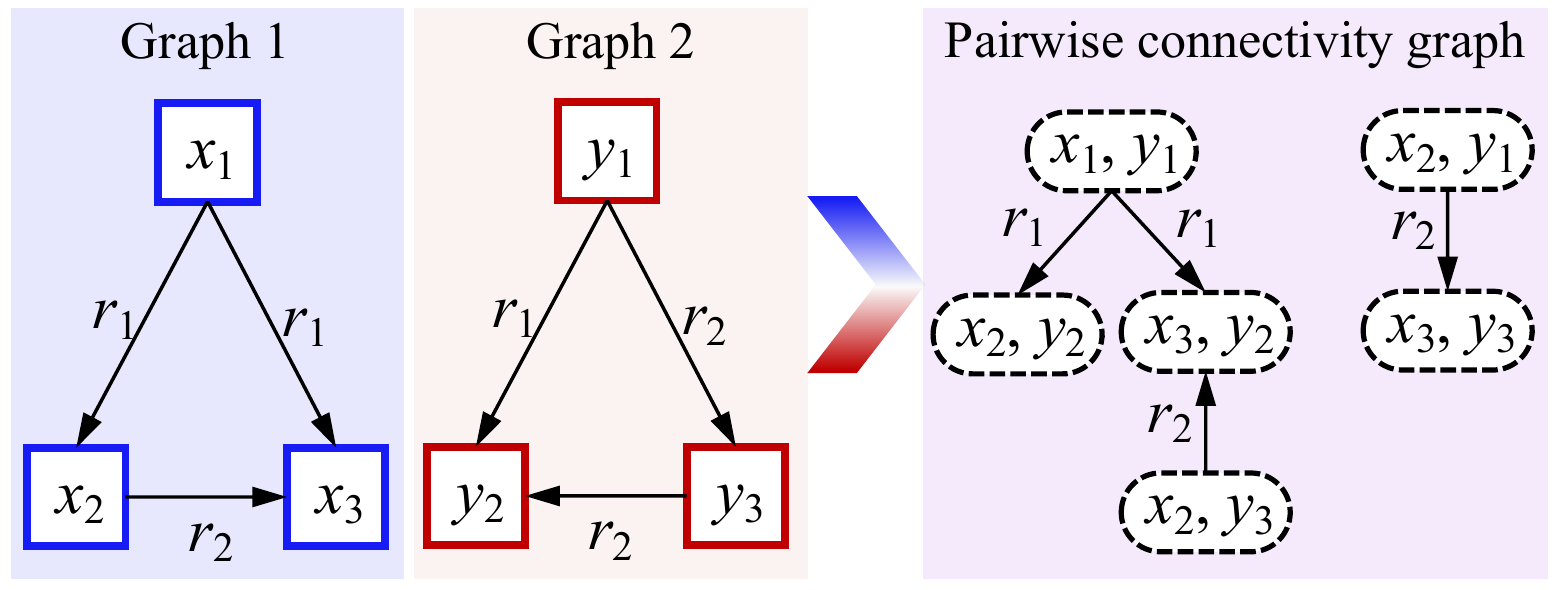}
\caption{Illustration of how to build the pairwise connectivity graph given two graphs (redrawn based on \cite{Similarity_flooding}).}
\label{fig:pcg}
\end{figure}

\begin{remark}
The pairwise connectivity graph construction requires the alignment of edge labels in the two graphs.
\end{remark}

\begin{remark}
The propagation coefficients of edges in the pairwise connectivity graph are computed heuristically.
\end{remark}

\section{Connecting Embedding-based EA and SF}\label{sect:connct_ea_sf}
Given the derived entity representations from TransE or GCN, we can compute entity similarities.
Specifically, given two entity sets, $\mathcal{X}=\{x_1, x_2, \dots, x_n\}$ and $\mathcal{Y}=\{y_1, y_2, \dots, y_m\}$,
we denote the derived entity representations by $\{\mathbf{x}_1, \mathbf{x}_2, \dots, \mathbf{x}_n\}$ and $\{\mathbf{y}_1, \mathbf{y}_2, \dots, \mathbf{y}_m\}$, respectively.
Their pairwise similarity matrix is
\begin{equation}\label{eq:sim_mat}
\Omega = (\mathbf{x}_1;\mathbf{x}_2;\dots;\mathbf{x}_n)^\top(\mathbf{y}_1;\mathbf{y}_2; \dots;\mathbf{y}_m)\in \mathbb{R}^{n\times m}.
\end{equation}
The similarity matrix determines entity alignment pairs.

\subsection{Unifying TransE- and GCN-based EA}

\begin{theorem} \label{th:theorem}
The TransE-based EA model seeks a fixpoint of pairwise entity similarities via embedding learning.
\end{theorem}

\begin{proof}
Eq.~(\ref{eq:transe_ent_embed_final}) shows that we can represent an entity with a composition of other entities.
Thus, we can represent an entity (taking $x_i\in \mathcal{X}$ for example) as
\begin{equation}\label{eq:ent}
\mathbf{x}_i = \lambda_{i,1} \mathbf{x}_1 + \lambda_{i,2} \mathbf{x}_2 + \cdots + \lambda_{i,n} \mathbf{x}_n = \sum_{k=1}^n \lambda_{i,k} \mathbf{x}_k,
\end{equation}
where $\lambda_{i,k}$ denotes the composition coefficient of entities $x_i$ and $x_k$, which can be computed from Eq.~(\ref{eq:transe_ent_embed_final}). 
Then, the similarity of entities $x_i\in \mathcal{X}$ and $y_j\in \mathcal{Y}$ can be calculated using the inner product\footnote{\scriptsize{The inner product of two normalized vectors is equal to their cosine similarity.}} as follows:
\begin{equation}\label{eq:sim_xy}\small
\resizebox{.99\columnwidth}{!}{$
\omega_{i,j} = \mathbf{x}_i \cdot  \mathbf{y}_j =  \sum_{k=1}^n \sum_{l=1}^m  \lambda_{i,k} \lambda'_{j,l} \, \mathbf{x}_k  \cdot \mathbf{y}_l =\sum_{k=1}^n \sum_{l=1}^m  \lambda_{i,k} \lambda'_{j,l}  \omega_{k,l}.
$}
\end{equation}
We can see from the above equation how the similarity of two entities affects their related neighbors.
Let the matrix $\Lambda=(\lambda_{i,j})_{i=1,j=1}^{n,n}$ consist of the lambda values for the source KG, and $\Lambda'=(\lambda'_{i,j})_{i=1,j=1}^{m,m}$ for the target KG.
Let $\Omega=(\omega_{i,j})_{i=1,j=1}^{n,m}$ denote the pairwise entity similarities of the two KGs.
We can rewrite Eq.~(\ref{eq:sim_xy}) as 
\begin{equation}\label{eq:sim_fixpoint}
\Lambda\Omega(\Lambda')^{\top}= \Omega,
\end{equation}
where $(\Lambda')^\top$ is the transposed matrix of $\Lambda'$.
This equation shows that the entity embedding similarities learned by the translation-based model have a fixpoint of $\Omega$.
\end{proof}

Next, we can connect aggregation-based EA models and similarity flooding in a similar way.
\begin{theorem} \label{th:theorem_gcn}
The GCN-based EA model seeks a fixpoint of pairwise entity similarities via embedding learning.
\end{theorem}

\begin{proof}
Please refer to the proof of Theorem~\ref{th:theorem}. The difference lies in how to compute the lambda values.
\end{proof}

We show how to calculate lambda values below.

\noindent\textbf{Lambda values for TransE.}
The lambda values for TransE are computed by counting the number of related triples:
\begin{equation}
\resizebox{.89\columnwidth}{!}{$
\lambda_{i,j} = \frac{1}{|\mathcal{T}_{x_i}|} \Big(|R(x_i,x_j)| + \sum_{r\in R}\frac{|\mathcal{T}_{x_i,r}|}{|\mathcal{T}_{r}|} \big(|\mathcal{T}_{x_j,r}| - |\mathcal{T}_{x_j,r^{-1}}|\big) \Big),
$}
\end{equation}
where $R(x_i,x_j)$ denotes the set of relations that connect entities $x_i$ and $x_j$, and $R$ is the set of all relations.
$\mathcal{T}_{x_i,r}$ denotes the set of relation triplets with $x_i$ as the subject and $r$ as the relation.
$\mathcal{T}_{r}$ is the set of relation triplets with $r$ as the relation.
$r^{-1}$ is the reverse relation for $r$.

\noindent\textbf{Lambda values for GCN.}
For GCN, we have
\begin{equation}
\lambda_{i,j} = \frac{\mathbbm{1}_{(x_i, r, x_j)\in \mathcal{T}}}{|\mathcal{T}_{x_i}|},
\end{equation}
where $\mathbbm{1}$ is an indicator function that returns 1 if there is a relation between $x_i$ and $x_j$, and 0 otherwise.

\begin{remark}
Embedding learning is just a means and the objective is to seek a fixpoint of pairwise entity similarities.
\end{remark}

\subsection{An Interpretation of Embedding-based EA}
Given the fixpoint view of EA, we further discover a mathematical interpretation of TransE- and GCN-based models.
We show that identical entities have isomorphic structures in the entity compositions of embedding-based EA.

\begin{theorem}\label{thm:mapping}
The entity alignment pairs found by the above embedding-based models yield a function $f: \{1,2,\dots,n\}\rightarrow \{0,1,2,\dots,m\}$, such that $\forall i,j, f(i)> 0\wedge f(j)>0\rightarrow\lambda'_{f(i),f(j)} \approx \lambda_{i,j}$.
\end{theorem}

\begin{proof}
Let us consider aligning $\mathcal{Y}$ with itself. 
We have
\begin{equation}
\Lambda' \mathbf{I}_m (\Lambda')^\top \approx \mathbf{I}_m,
\end{equation}
where $\mathbf{I}_m$ is an identity matrix.
Suppose that the alignment found by the above embedding-based models is $\hat{\mathcal{A}}$, which can be denoted by a 0-1 matrix $\hat{\Omega}$ such that $\hat{\omega}_{i,j} = 1$ if and only if $(x_i,x_j)\in\hat{\mathcal{A}}$.
Similar to most EA settings, we assume that in $\hat{\mathcal{A}}$, each entity is aligned to at most one entity in another KG.
Notice that $\hat{\Omega}$ approximately equals a fixpoint of Eq. (\ref{eq:fixpoint}).
Thus, we have 
\begin{align}
\hat{\Omega}^\top \Lambda \hat{\Omega} (\Lambda')^\top\approx \hat{\Omega}^\top \hat{\Omega}=\hat{\mathbf{I}}_m,
\end{align}
where $\hat{\mathbf{I}}_m$ is a diagonal matrix, where $\hat{\mathbf{I}}_{j,j}=1$ if and only if $y_j$ appears in one pair in $\hat{\mathcal{A}}$.
As $\Lambda' \mathbf{I}_m (\Lambda')^\top \approx \mathbf{I}_m$, we have $\hat{\Omega}^\top \Lambda \hat{\Omega}\approx \hat{\mathbf{I}}_m \Lambda'$.
Let $f$ be a function defined as 
\begin{equation}
f(i) = \begin{cases}
j, & (x_i,y_j) \in \hat{\mathcal{A}} \\
0, & \forall y_j\in\mathcal{Y}, (x_i,y_j)\not\in \mathcal{A}
\end{cases}.
\end{equation}
When $f(i) > 0$ and $f(j) > 0$, we have $(\hat{\Omega}^\top \Lambda \hat{\Omega})_{f(i),f(j)}=\lambda_{i,j}$, i.e., $\lambda'_{f(i),f(j)}\approx\lambda_{i,j}$.
\end{proof}

Based on \autoref{thm:mapping}, we find that for each KG, the entity compositions derived from EA models generate a matrix (e.g., $\Lambda$) that only depends on graph structures.
It finds a mapping function that makes the two KGs' matrices the same and this function determines the alignment results.
Although different KGs may have heterogeneous structures, 
the entity compositions in embedding-based EA models reconstruct a new structure (represented by $\Lambda$), in which aligned entities have isomorphic subgraphs.

\begin{remark}
If we view these matrices as edge weights between nodes in KGs, these embedding-based EA models mathematically conduct graph matching.
\end{remark}

\section{Experimental Evidence}\label{sect:exp1}
In this section, we propose two methods to improve EA: similarity flooding via entity compositions and self-propagation in GCNs.
We evaluate them on benchmark datasets, providing experimental evidence to support our theorem.

\subsection{Similarity Flooding via Entity Compositions}\label{sect:embed_sf}
We have shown by Eqs.~(\ref{eq:transe_ent_embed_final}) and (\ref{eq:gcn_e}) that the entity representations derived from TransE- and GCN-based models can be reformulated to be independent from relations.
Our theorems in Section~\ref{sect:connct_ea_sf} show that entity similarities are determined by other entity similarities, and entity embeddings are unnecessary in this computation.
Then, a natural question arises: \textit{Is the embedding learning process a prerequisite for achieving the fixpoint of entity similarities}?

Given Eq.~(\ref{eq:sim_fixpoint}), we design a similarity flooding style algorithm to propagate the entity similarities that are computed based on the entity composition representations induced from an embedding model.
It is presented in Algorithm~\ref{alg:embed_sf}.
We first derive the entity compositions from the embedding model.
Then, we calculate the lambda values $\Lambda$ and $\Lambda'$ in the compositions.
We initialize the similarity matrix $\Omega$ to be a zero matrix and set the values that indicate seed EA similarities to be 1.
The similarity matrix is further updated to achieve the fixpoint as shown in Eq.~(\ref{eq:sim_fixpoint}). 
After each update, we normalize the values to range from $-1$ to $1$.
The computation is performed in an iterative manner until it converges or reaches the maximum number of iterations.

\begin{algorithm}[t]\small
\caption{Similarity flooding via entity compositions}
\label{alg:embed_sf}
\begin{algorithmic}
\STATE {\bfseries Input:} {$KG_1$, $KG_2$, seed entity alignment $\mathcal{A}$, the maximum number of iterations  $T$, a small threshold value $\epsilon$ for algorithm termination, an embedding model $\mathcal{M}$.}
\WHILE {true}
\STATE Derive the entity compositions from M;
\STATE Compute lambda matrices $\Lambda$ and $\Lambda'$ in the compositions;
\STATE $\Omega_0 \leftarrow (0)_{i=1,j=1}^{n,m}$;
\FOR{$i,j\in \mathcal{A}$}
\STATE ${\Omega_0}_{\,i,j} \leftarrow 1$;
\ENDFOR
\FOR{$t=1,2,\dots,T$}
\STATE $\Omega_t  \leftarrow \texttt{normalize}\big(\Lambda\Omega_{t-1}(\Lambda')^\top\big)$; 
\ENDFOR
\IF{$\bigtriangleup(\Omega_t, \Omega_{t-1}) < \epsilon$}
\STATE \textbf{break}\;
\ENDIF
\ENDWHILE
\end{algorithmic}
\end{algorithm}

We hereby discuss the advantages of the algorithm.
First, it does not need relation alignment.
It represents an entity without using relations and does not need to build the PCG.
As different KGs usually have heterogeneous schemata, it is difficult to obtain the accurate relation alignment.
By contrast, the conventional similarity flooding algorithm relies on relation alignment to build the PCG.
Second, our algorithm does not need to compute propagation coefficients for similarity flooding. 
The lambda values act as ``propagation coefficients'', but they are calculated by counting the number of related triplets without using heuristic methods.
Third, our algorithm does not need to learn embeddings, but it needs an embedding model to derive the entity compositions.
Our optimization objective is to directly achieve the fixpoint of entity similarities.
Embedding-based models seek this goal by an indirect way of updating embeddings.

Our algorithm has the disadvantage of requiring matrix manipulation.
If the KG scale is large, it would consume a lot of memory. 
We can solve this problem by using advanced and parallel matrix manipulation implementations.

\begin{table*}[!t]
\centering
\caption{EA results on DBP15K as well as OpenEA D-W and D-Y. The best scores in each group are marked in bold. The results of MTransE are taken from \cite{JAPE}. The results of GCN-Align are taken from its paper. ``-'' denotes their unreported metrics.}
\vskip 0.15in
\resizebox{0.999\textwidth}{!}{
\begin{tabular}{lccccccccccccccc}
\toprule
\multirow{2}{*}{Models} & \multicolumn{3}{c}{DBP15K ZH-EN} & \multicolumn{3}{c}{DBP15K JA-EN} & \multicolumn{3}{c}{DBP15K FR-EN} & \multicolumn{3}{c}{{OpenEA D-W 15K}} & \multicolumn{3}{c}{{OpenEA D-Y 15K}} \\
\cmidrule(lr){2-4} \cmidrule(lr){5-7} \cmidrule(lr){8-10} \cmidrule(lr){11-13} \cmidrule(lr){14-16}
& Hits@1 & Hits@10 & MRR & Hits@1 & Hits@10 & MRR & Hits@1 & Hits@10 & MRR & {Hits@1} & {Hits@10} & {MRR} & {Hits@1} & {Hits@10} & {MRR} \\ 
\midrule
MTransE & 0.308 & 0.614 & - & 0.279 & 0.575 & - & 0.244 & 0.556 & - & {0.259} & {-} & {0.354} & {0.463} & {-} & {0.559}\\
\transmodelname (ours) & \textbf{0.315} & \textbf{0.707} & 0.451 & \textbf{0.372} & \textbf{0.757} & 0.505 & \textbf{0.347} & \textbf{0.752} & 0.484 & {\textbf{0.294}} & {0.699} & {\textbf{0.427}} & {\textbf{0.503}} & {0.880} & {\textbf{0.641}}\\
\midrule
GCN-Align & \textbf{0.413} & 0.744 & - & \textbf{0.399} & 0.745 & - & \textbf{0.373} & 0.745 & - & {\textbf{0.364}} & {-} & {0.461} & {0.465} & {-} & {0.536}\\
\gcnfloodmodelname (ours) &  0.349 & \textbf{0.761} & 0.490 & 0.376 & \textbf{0.770} & 0.512 & 0.349 & \textbf{0.761} & 0.490 & {0.358} & {0.739} & {\textbf{0.486}} & {\textbf{0.478}} & {0.754} & {\textbf{0.583}} \\
\bottomrule
\end{tabular}
}
\label{tab:entity_align_results_sf}
\end{table*}

\subsubsection{Evaluation}\label{sect:exp2}
We implement two variants of our algorithm, namely \transmodelname and \gcnfloodmodelname, using TransE and GCN, respectively.

\noindent\textbf{Baselines.}\label{sect:sf_baseline}
We choose the translation-based model MTransE and aggregation-based model GCN-Align as baselines.
\begin{itemize}
\setlength{\itemsep}{1pt}
\item \textbf{MTransE} \cite{MTransE} is one of the earliest studies that explore translational embeddings for EA. It uses TransE \cite{TransE} to learn the entity embeddings of two KGs meanwhile learning a linear mapping to find identical entities. 
\item \textbf{GCN-Align} \cite{GCNAlign} is the first work that considers GCNs for KG EA.
It employs the vanilla GCN \cite{GCN} to generate entity embeddings and uses the marginal ranking loss with uniform negative sampling for alignment learning.
\end{itemize}

\noindent\textbf{Datasets.}
We consider two datasets in our experiment.
One is the widely-used dataset DBP15K \cite{JAPE} 
It aims to align the cross-lingual entities extracted from  DBpedia \cite{DBpedia}.
It has three EA settings: ZH-EN (Chinese-English), JA-EN (Japanese-English) and FR-EN (French-English).
The triples in these KGs are extracted from the infobox data of multilingual Wikipedia.
They have similar rather than identical schemata because the data is not mapped to a unified ontology.
Each setting has $15,000$ pairs of identical entities for alignment learning and test.
We follow the data splits of DBP15K and use $30\%$ of entity alignment as training data.
The other dataset is OpenEA~\cite{OpenEA} and we choose its 15K V1 versions of D-W (DBpedia-Wikidata) and D-Y (DBpedia-YAGO), in each of which the two KGs have different schemata.
Each setting also has $15,000$ entity alignment pairs and we follow its data splits and use $20\%$ of entity alignment as training data.

\noindent\textbf{Metrics.}
Following the conventions, we choose Hits@$k$ ($k=1,10$) and mean reciprocal rank (MRR) as metrics to assess EA performance. 
Hits@$k$ measures the proportion of correctly-aligned entities ranked in the top $k$. 
MRR is the average of the reciprocal ranks. 
Higher Hits@$k$ and MRR scores indicate better performance. 

\noindent\textbf{Main results.}
We present the results in Table~\ref{tab:entity_align_results_sf}.
We can observe that the proposed \transmodelname achieves much better performance than MTransE on all datasets.
For example, on FR-EN, the Hits@1 score of \transmodelname is $0.347$, outperforming MTransE by $0.103$.
We find that, as a learning model, MTransE is easy to overfit.
Our model is also derived from TransE but our iteration algorithm can enable our model to get a more stable solution than the learning method.
For aggregation-based EA, our \gcnfloodmodelname achieves comparative Hits@1 results and better Hits@10 scores compared to GCN-Align.
Our \gcnfloodmodelname only considers one-hop neighbors to generate entity similarities (i.e., a one-layer GCN), whose information is less than that in GCN-Align (a two-layer GCN).
However, its advantage lies in that it converges directly to the fixpoint, while the embedding learning method cannot guarantee this.
Overall, \transmodelname and \gcnfloodmodelname that do not need learning can achieve comparable or even better performance than embedding learning baselines.

\noindent\textbf{Running time comparison.}
We compare the running time of our algorithm variants against MTransE and GCN-Align on ZH-EN.
This experiment is conducted using a personal workstation with an Intel Xeon E3 3.3GHz CPU, 128GB memory and a NVIDIA GeForce GTX 1080Ti GPU.
The results are shown in Figure~\ref{fig:running_time}. We observed similar results on the other two datasets.
MTransE uses the least time because it is a shallow model that can be easily optimized.
GCN-Align takes the most time. We find that it converges very slowly and takes many training epochs.
Our \transmodelname and \gcnfloodmodelname take very similar time, which is also less than that of GCN-Align.
In our algorithm, resolving Eq.~(\ref{eq:sim_fixpoint}) costs the most in training time.
Overall, our algorithm, which does not need to learn embeddings, can achieve comparable or even better performance in both effectiveness and efficiency than embedding learning models.

\begin{figure}[!ht]
\centering
\includegraphics[width=0.6\linewidth]{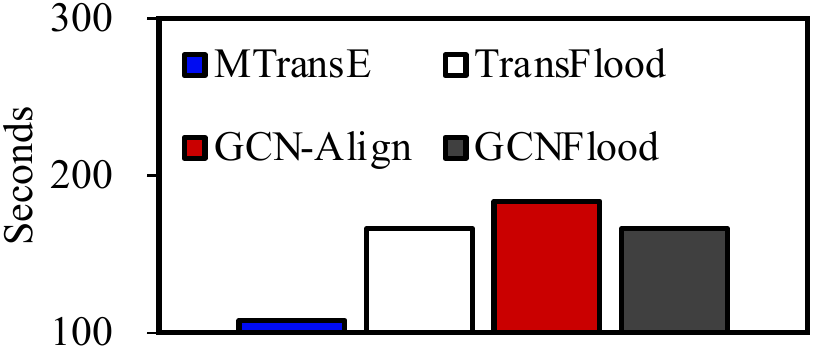}
\caption{Total running time (in seconds) on ZH-EN.}
\label{fig:running_time}
\end{figure}

\begin{table*}[!ht]
\centering
\caption{EA results using text features on DBP15K.}
 \vskip 0.15in
\resizebox{0.75\textwidth}{!}{
\begin{tabular}{lccccccccc}
\toprule
\multirow{2}{*}{Models} & \multicolumn{3}{c}{ZH-EN} & \multicolumn{3}{c}{JA-EN} & \multicolumn{3}{c}{FR-EN} \\
\cmidrule(lr){2-4} \cmidrule(lr){5-7} \cmidrule(lr){8-10}
& Hits@1 & Hits@10 & MRR & Hits@1 & Hits@10 & MRR & Hits@1 & Hits@10 & MRR \\ 
\midrule
RDGCN~\cite{RDGCN} & 0.708 & 0.846 & - & 0.767 & 0.895 & - & 0.886 & 0.957 & - \\
\midrule
\transmodelname & {0.315} & {0.707} & 0.451 & {0.372} & {0.757} & 0.505 & {0.347} & {0.752} & 0.484 \\
\gcnfloodmodelname &  0.349 & {0.761} & 0.490 & 0.376 & {0.770} & 0.512 & 0.349 & {0.761} & 0.490  \\
\midrule
\transmodelname + Text & 0.670 & 0.786 & 0.713 & 0.747 & 0.868 & 0.794 & 0.881 & 0.949 & 0.908 \\
\gcnfloodmodelname + Text & 0.651 & 0.823 & 0.716 & 0.712 & 0.882 & 0.777 & 0.842 & 0.957 & 0.887 \\
\bottomrule
\end{tabular}}
\label{tab:entity_align_results_text}
\end{table*}

\noindent\textbf{Results using text features.}\label{appx:text}
Our similarity flooding algorithm can also use text features to improve performance.
We use multilingual word embeddings \cite{fastText} to encode entity names for computing the similarity matrix, which is further combined with $\Omega$ in our Algorithm~\ref{alg:embed_sf}.
We conduct experiments on DBP15K and present the results in Table~\ref{tab:entity_align_results_text}.
We choose RDGCN~\cite{RDGCN} as a baseline.
We can see that our \transmodelname + Text and \gcnfloodmodelname + Text achieve slightly lower results than RDGCN.
On FR-EN, \transmodelname + Text achieves comparable results with RDGCN.
Moreover, by using text features, both \transmodelname and \gcnfloodmodelname get greatly improved.
These results show the generalization ability of our algorithm.

\subsection{Self-propagation in Neighborhood Aggregation}\label{sect:spa}
Based on our theoretical analysis of embedding-based EA and similarity flooding, 
we derive a new aggregation scheme for EA: self-propagation and neighbor aggregation.

As previously stated, an embedding-based EA model aims to establish a fixpoint of pairwise entity similarities by updating entity embeddings throughout the training process.
Considering that entity similarities are computed using entity embeddings,
the output of GCNs also achieves a fixpoint.
We can rewrite neighborhood aggregation as
\begin{equation} 
\label{eq:sf_gcn}
\mathbf{e} = f\big(\mathbf{e}, \oplus_{z \in N_e}(\mathbf{z})\big),
\end{equation}
which means that the entity embeddings remain ``unchanged'' after aggregation.
For brevity, we use the function $G()$ to denote a GCN layer and consider a two-layer GCN.
Given input embedding $\mathbf{e}^0$, in the fixpoint, we expect to hold
\begin{equation} 
\label{eq:gcn_fixpoint}
\mathbf{e}^2 = G(\mathbf{e}^1) = \mathbf{e}^1 = G(\mathbf{e}^0) = \mathbf{e}^0.
\end{equation}
However, this equation has limitations.
First, in this case, the aggregation function degenerates into an identity mapping.
Second, it almost loses the neighborhood information.
To resolve the issues, inspired by \cite{PredictPropagate}, we enable the GCN output to have a probability of backing to the input.
The aggregation function is rewritten as:
\begin{equation} 
\label{eq:my_gcn}
\mathbf{e}^{i+1} = (1-\alpha)\oplus_{z \in N_e}(\mathbf{z}) + \alpha f(\mathbf{e}^{i}),
\end{equation}
where $\alpha$ is a hyper-parameter indicating the probability of backing to the input.
Here, we use $f()$ to denote a dense layer.
$\mathbf{e}^0$ is randomly initialized as the input embedding of entity $e$.
The output of the GCN, i.e., $\mathbf{e}^2$ for a two-layer GCN, is used for alignment learning and search.
Please note that, although the conventional GCNs also consider the entity itself in neighbor aggregation, the work \cite{PredictPropagate} shows that they would still lose the local focus of the entity itself during layer-by-layer aggregation.

\subsubsection{Properties of Self-propagation}
Taking a deep learning perspective, we find that the proposed self-propagation has several good properties.

\noindent\textbf{Model complexity.}
The proposed self-propagation can be easily combined with any aggregation function without adding additional computational complexity.
It only introduces a dense layer for feature transformation.
Considering that the number of entity embedding parameters is much larger than that of a dense layer,
we argue that the parameter complexity of self-propagation remains similar to that of other aggregation functions.

\noindent\textbf{Relation to PageRank-based GCNs.}
PageRank-based GCNs introduce the possibility of resetting the neighborhood aggregation to its initial state during training \cite{PredictPropagate,PageRankGCN}.
These studies are relevant to random walks with restarts on graphs where the random walk has a probability of backing to the start node after several steps.
The idea is similar to ours.
The difference is that we do not seek the representation of an entity to return to itself after several times of neighborhood aggregation.
Instead, we seek to increase the local focus on the entity representation itself within the iterative neighborhood aggregation.
Self-propagation is also helpful to resolve the over-smoothing issue.

\noindent\textbf{Relation to residual learning.}
The self-propagation can be regarded as a special case of residual learning \cite{Residual} because it builds a skipping connection between two GCN layers.
Given the input $\mathbf{x}$, let $F(\mathbf{x})$ be a representation function (e.g., the $G()$ in our paper), and $H(\mathbf{x})$ be the expected output representation. 
Residual learning indicates that directly optimizing $F(\mathbf{x})$ to fit $H(\mathbf{x})$ is more difficult than letting $F(\mathbf{x})$ fit the residual part $H(\mathbf{x})-\mathbf{x}$.
For aggregation-based EA, we cannot let $H(\mathbf{x})=\mathbf{x}$, in which case the function $F()$ has no representation ability.
Therefore, we introduce the transformation function $f()$, and let $G(\mathbf{x})$ fit $H(\mathbf{x})-f(\mathbf{x})$.
A related work \cite{RSN} shows that the skipping connection would also improve the optimization of KG embeddings.

\subsubsection{Evaluation}
We present our experimental results 
on DBP15K and OpenEA in terms of Hits@1, Hits@10 and MRR scores.

\noindent\textbf{Implementation.}\label{sect:self_imp}
The performance of an EA model relates to not only the embedding learning model (e.g., TransE or GCN) but also other modules, including the alignment learning loss, the negative sampling method, and even the tricks in deep learning such as the parameter initialization method and the loss optimizer.
To study the real effectiveness of self-propagation, we do not develop a new aggregation-based model from scratch. Instead, we choose four representative aggregation-based models: GCN-Align (see Section~\ref{sect:sf_baseline}), AliNet, Dual-AMN and RoadEA, and incorporate self-propagation into them to see performance changes.

\begin{table*}[!t]
\centering
\caption{EA results on DBP15K as well as OpenEA D-W and D-Y. The best scores in each group are marked in bold. The results of baseline models are taken from their papers, respectively, and ``-'' denotes the unreported metric in the corresponding original paper. The results of RoadEA on DBP15K and D-Y are produced using its code. Its results on D-W are taken from its paper.}
\vskip 0.15in
\resizebox{0.999\textwidth}{!}{\large
\begin{tabular}{lccccccccccccccc}
\toprule
\multirow{2}{*}{Models} & \multicolumn{3}{c}{DBP15K ZH-EN} & \multicolumn{3}{c}{DBP15K JA-EN} & \multicolumn{3}{c}{DBP15K FR-EN} & \multicolumn{3}{c}{{OpenEA D-W 15K}} & \multicolumn{3}{c}{{OpenEA D-Y 15K}} \\
\cmidrule(lr){2-4} \cmidrule(lr){5-7} \cmidrule(lr){8-10} \cmidrule(lr){11-13} \cmidrule(lr){14-16}
& Hits@1 & Hits@10 & MRR & Hits@1 & Hits@10 & MRR & Hits@1 & Hits@10 & MRR & {Hits@1} & {Hits@10} & {MRR} & {Hits@1} & {Hits@10} & {MRR} \\ 
\midrule
GCN-Align & 0.413& 0.744 & - & 0.399 & 0.745 & - & 0.373 & 0.745 & - & {{0.364}} & {-} & {0.461} & {0.465} & {-} & {0.536}\\
GCN-Align + \gcnmodelname (ours) & \textbf{0.441} & \textbf{0.751} & - & \textbf{0.446} & \textbf{0.759} & - & \textbf{0.414} & \textbf{0.763} & - & {\textbf{0.378}} & {{0.628}} & {\textbf{0.464}}  & {\textbf{0.495}} & {\textbf{0.688}} & {\textbf{0.565}} \\
\midrule
AliNet & 0.539 & 0.826 & 0.628 & 0.549 & \textbf{0.831} & 0.645 & 0.552 & 0.852 & 0.657 & {0.440} & {\textbf{0.672}} & {0.522} & {0.559} & {\textbf{0.713}} & {0.617} \\
AliNet + \gcnmodelname (ours) & \textbf{0.575} & \textbf{0.829} & \textbf{0.664} & \textbf{0.570} & 0.821 & \textbf{0.658} & \textbf{0.581} & \textbf{0.857} & \textbf{0.678} & {\textbf{0.451}} & {0.668} & {\textbf{0.529}}  & {\textbf{0.563}} & {0.702} & {\textbf{0.624}} \\
\midrule
Dual-AMN & 0.731 & 0.923 & 0.799 & 0.726 & 0.927 & 0.799 & 0.756 & 0.948 & 0.827 & {0.683} & {0.893} & {0.761} & {0.767} & {0.908} & {0.823} \\
Dual-AMN + \gcnmodelname (ours) & \textbf{0.733} & \textbf{0.925} & \textbf{0.804} & \textbf{0.735} & \textbf{0.936} & \textbf{0.807} & \textbf{0.767} & \textbf{0.951} & \textbf{0.835} & {\textbf{0.695}} & {\textbf{0.898}} & {\textbf{0.771}}  & {\textbf{0.779}} & {\textbf{0.912}} & {\textbf{0.832}} \\
\midrule
RoadEA & {0.570} & {{0.848}} & {0.667} & {0.569} & {0.857} & {0.669} & {0.578} & {0.875} & {0.680} & {0.495} & {-} & {0.584} & {0.212} & {0.296} & {0.244}\\
RoadEA + \gcnmodelname (ours) & {\textbf{0.579}} & {\textbf{0.849}} & {\textbf{0.673}} & {\textbf{0.577}} & {\textbf{0.858}} & {\textbf{0.676}} & {\textbf{0.593}} & {\textbf{0.876}} & {\textbf{0.690}} & {\textbf{0.502}} & {0.744} & {\textbf{0.587}} & {\textbf{0.235}} & {\textbf{0.309}} & {\textbf{0.261}}\\
\bottomrule
\end{tabular}}
\label{tab:ent_alignment_gcn}
\end{table*}

\begin{itemize}
\setlength{\itemsep}{0pt}
\setlength{\topsep}{0pt}
\item \textbf{AliNet} \cite{AliNet} extends GCN-Align by introducing distant neighbors in the aggregation function.
Its learning objective is to minimize the limit-based loss with truncated negative sampling \cite{BootEA}.
It concatenates the output of multiple layers as representations for alignment learning and search.
\item \textbf{Dual-AMN} \cite{Dual-AMN} is the state-of-the-art aggregation-based model according to our knowledge.
It designs several advanced implementations, including the proxy matching attention, normalized hard sample mining and loss normalization.
It achieves prominent performance in both effectiveness and efficiency.
\item \textbf{RoadEA} \cite{RoadEA} is a recent GCN-based EA method that considers relations in neighborhood aggregation. It combines relation embeddings and their corresponding neighbor embeddings as relation-neighbor representations and uses graph attention networks \cite{GAT} to aggregate them. 
\end{itemize}

For each baseline model, we adopt its official code and incorporate the proposed self-propagation into its aggregation function.
To be specific, in each of their layers, we add a self-propagation connection between their input and output.
We leave other modules, including the alignment learning loss, the negative sampling method, and the alignment search strategy, unchanged.
As a result, we get four GCN-based model variants, namely ``GCN-Align + \gcnmodelname'', ``AliNet + \gcnmodelname'', ``Dual-AMN + \gcnmodelname'', and ``RoadEA + \gcnmodelname''.

\noindent\textbf{Settings.}
To ensure a fair comparison, the hyper-parameter values in our experiment follow the default settings of the corresponding baselines.
The only exception is that the embedding dimensions of the input and two GCN layers in AliNet+SP are $384$, $384$ and $384$, respectively, which are different from the original settings of $500$, $400$ and $300$ in AliNet.
The reason that we keep these layers with the same output dimension is that we can directly compare the representations of an entity in different AliNet layers (see Section \ref{sect:layer_comparison}).
Note that AliNet concatenates the output of all layers as the final entity representations for alignment learning and search.
In our model, the final embedding dimension is $384+384+384=1152$, slightly smaller than that of AliNet ($500+400+300=1200$).
We find that such a small dimension difference has no observed impact on performance.
In our models, $\alpha=0.1$ for all datasets.

\noindent\textbf{Main results.}
Table \ref{tab:ent_alignment_gcn} presents the EA results of baselines and our model variants on DBP15K.
We can see that our model variants can bring stable improvement on DBP15K, especially on Hits@$1$ and MRR, compared with the corresponding baselines.
For example, AliNet+\gcnmodelname outperforms AliNet by $0.036$ on Hits@1.
Even when compared to the state-of-the-art model Dual-AMN, 
our Dual-AMN+\gcnmodelname still achieves higher performance, especially on JA-EN and FR-EN, establishing a new state-of-the-art.
As we have discussed in Section~\ref{sect:self_imp}, Dual-AMN has many advanced designs to improve performance.
Boosting its performance to a higher level is much more difficult than that for GCN-Align and AliNet.
We find that RoadEA fails to achieve promising results on D-Y.
We think this is because DBpedia and YAGO have an unbalanced number of relations, which affects the relational attention mechanism in RoadEA.
However, our self-propagation still improves it, showing good robustness.
To summarize, this comparison demonstrates the effectiveness and generalization of the proposed self-propagation for EA.
We conduct additional experiments in the following two subsections to further investigate the reasons for the good performance of self-propagation.

\noindent\textbf{Effectiveness of self-propagation against over-smoothing.} 
The over-smoothing issue of GCNs refers to the fact that the output representations tend to be similar if too many layers are used for neighborhood aggregation \cite{OverSmoothing,MeasureOverSmoothing}.
It is obvious that such an issue has a negative impact on embedding-based EA.
The default settings of GCN layer numbers in GCN-Align, AliNet and Dual-AMN are all $2$.
To investigate the over-smoothing issue in EA,
we show in Figure~\ref{fig:layer_num} the Hits@1 results of these baselines (in blue) and our model variants (in red) on ZH-EN when their layer numbers are set as $1,2,3,4$, respectively.
Both GCN-Align and AliNet suffer from over-smoothing.
Their results decrease as the GCNs go deeper with more than two layers.
By adding the self-propagation connection, their performance degradation is reduced.
By contrast, Dual-AMN shows good robustness against over-smoothing.
Its performance changes little when the layer number increases.
Dual-AMN+\gcnmodelname also benefits from such robustness.
Dual-AMN uses the normalized hard sample mining method with a large number of negative examples, enabling dissimilar entities to have distinguishable representations.

\begin{figure}[!t]
\centering
\includegraphics[width=0.95\linewidth]{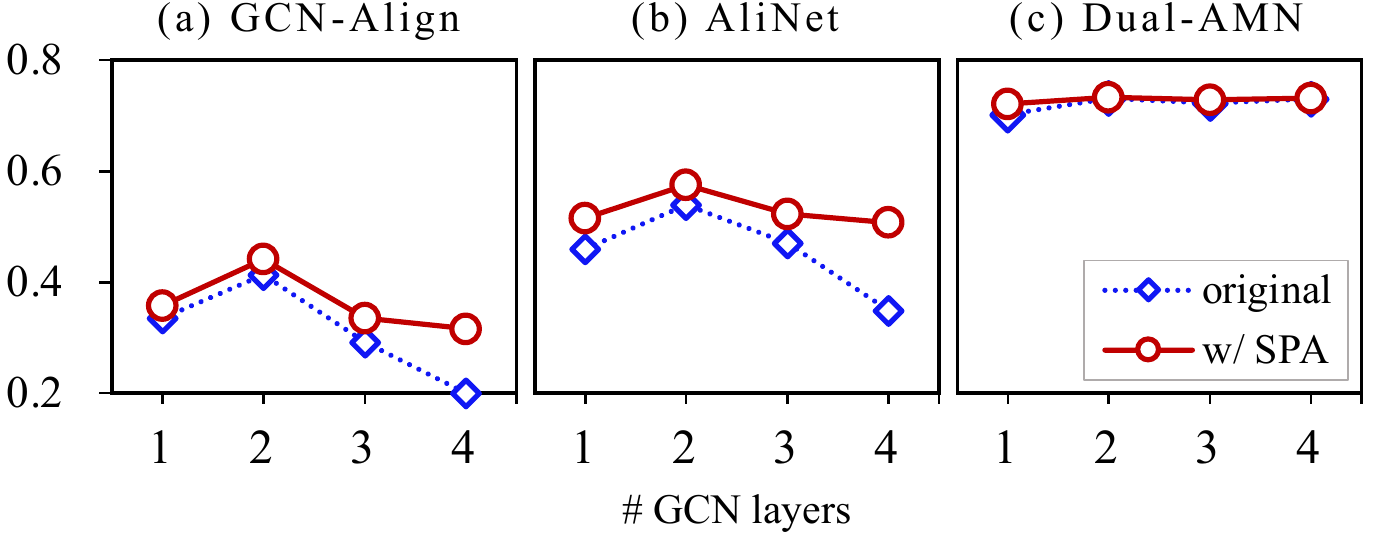}
\caption{Hits@1 on ZH-EN w.r.t. the number of GCN layers.}
\label{fig:layer_num}
\vspace{-5pt}
\end{figure}

\noindent\textbf{Layer output representation comparison.}\label{sect:layer_comparison}
We further compare the output representation distance of the last two layers in AliNet and AliNet+\gcnmodelname.
Figure~\ref{fig:layer_distance} shows the average Euclidean distance w.r.t. the first 80 training epochs on ZH-EN.
We can see that the output representation distance of both AliNet and AliNet+\gcnmodelname becomes smaller as the validation performance increases.
Furthermore, by adding the proposed self-propagation connection, the layer output distance of AliNet+\gcnmodelname is smaller than that of AliNet.
These results provide experimental evidence to support our design of self-propagation to connect two GCN layers and increase the local focus on the entity embedding itself.

\section{Related Work}\label{sect:related_work}
Our work is relevant to multi-sourced KG embedding learning and iteration-based node matching methods for graphs.

\subsection{Multi-sourced KG Embeddings}
Multi-sourced KG representation learning starts with the research on embedding-based EA.
An embedding-based EA model learns and measures entity embeddings to compute entity similarities.
It usually has two learning objectives.
One is for embedding learning and the other is for alignment learning.
Translation-based EA models \cite{MTransE,IPTransE,JAPE,TransEdge,SEA} adopt TransE \cite{TransE} or its variants for embedding learning.
Aggregation-based EA models adopt GNNs to generate entity embeddings,
including the vanilla GCNs \cite{GCNAlign}, multi-hop GCNs \cite{AliNet}, relational GCNs \cite{MultiRelationalGCN}, graph attention networks \cite{CG-MuAlign,Dual-AMN,RoadEA}, self-supervised GCNs \cite{SelfKG} and temporal GCNs \cite{TemporalRelEA}.
Our proposed self-propagation is a plug-in for GNNs.  It adds a direct connection between entity representations and the aggregated neighbor representations.
In addition to the above two types of basic models for EA, 
other studies consider using semi-supervised or active learning techniques to augment EA \cite{BootEA,KDCoE,UPLR,ActiveEA_ECIR,ActiveEA,RAC} or introduce some text features (e.g., entity names, attributes and descriptions) \cite{JAPE,AttrE,RDGCN} or temporal information \cite{TEA-GNN,TREA} to enhance embedding learning.
These studies are not relevant to our work.
Interested readers can refer to the survey \cite{EA_survey,EA_TKDE,EA_VLDBJ} for more details.
However, our work can also benefit from side features.

\begin{figure}[!t]
\centering
\includegraphics[width=0.61\linewidth]{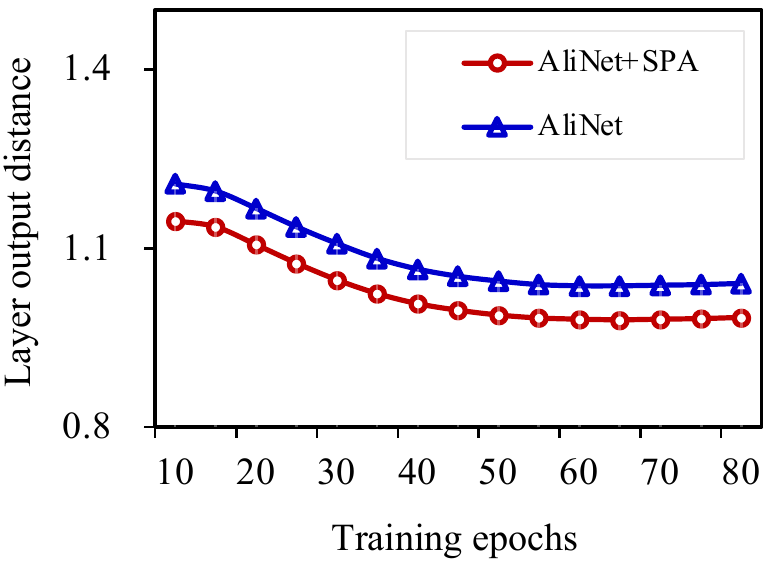}
\caption{Output representation distance of the last two layers in AliNet and AliNet+\gcnmodelname on ZH-EN.}
\label{fig:layer_distance}
\vspace{-5pt}
\end{figure}

\subsection{Iteration-based Graph Matching}
Computing node similarities in graphs is a long-standing research topic in many areas, such as databases.
Our work is relevant to iteration-based similarity computation methods, including similarity flooding \cite{Similarity_flooding}, SimRank \cite{SimRank} and NetAlignMP \cite{NetAlignMP}.
Their key assumption is that ``two nodes are similar if their neighbors are similar''.
They first compute the similarity of some pairs of nodes. 
Then, they propagate these similarities to other related node pairs using different heuristic rules iteratively, until they achieve a fixpoint of node pairwise similarities.
Our work shows that the embedding-based EA models follow the same key assumption as the conventional iteration-based graph alignment methods.
We build a connection between the two types of methods, which would help users acquire deep insights into them.

\section{Conclusions and Future Work}\label{sect:conclusion}
In this paper, we present a similarity flooding perspective to understand translation-based and aggregation-based EA models.
We prove that these models essentially seek a fixpoint of entity pairwise similarities through embedding learning.
Based on this finding, we propose two methods, i.e., similarity flooding via entity compositions and self propagation, for improving EA.
Experiments on benchmark datasets demonstrate their effectiveness.
Our work fills the gap between recent embedding-based EA and the conventional iteration-based graph matching.

We think there are two promising directions for future work.
The first is to develop neural-symbolic EA models that take advantage of both the representation learning ability of neural models and the interpretability of conventional symbolic methods.
The second is, given EA, to learn more expressive and transferable multi-sourced KG embeddings to improve downstream knowledge-enhanced tasks.
A KG-enhanced task can be extended into a multi-sourced KG-enhanced task.
The latter can benefit from the knowledge transfer in multi-sourced KGs and thus get further improvement.


\section*{Acknowledgments} 
This work is funded by the National Natural Science Foundation of China (No. 62272219) and the Alibaba Group through Alibaba Research Fellowship Program.


\bibliography{sample-base}
\bibliographystyle{icml2023}


\end{document}